\documentclass[11pt]{article}
\usepackage{theapa, rawfonts}

\usepackage[scale=0.7]{geometry}

\usepackage[utf8]{inputenc}
\usepackage{mathtools,amssymb,amsthm}
\usepackage[normalem]{ulem}
\usepackage{xspace}
\usepackage{colortbl}
\usepackage{booktabs}
\usepackage{multirow}
\usepackage{cases}
\usepackage{algorithm}
\usepackage{algpseudocode}
\usepackage{subcaption}
\usepackage{tabularx}
\usepackage{array}
\usepackage{url}
\usepackage{tikz}
\usepackage{tikzscale}
\usetikzlibrary{arrows,tikzmark,positioning}
\usetikzlibrary{decorations}
\usetikzlibrary{decorations.pathreplacing}

\DeclarePairedDelimiter{\floor}{\lfloor}{\rfloor}

\newcolumntype{R}{>{\raggedleft\arraybackslash}X}

\theoremstyle{definition}
\newtheorem{definition}{Definition}
\newtheorem{theorem}{Theorem}
\newtheorem*{example*}{Example}

\newtheoremstyle{break}
  {\topsep}{\topsep}%
  {}{}%
  {\bfseries}{}%
  {\newline}{\thmname{#1}\thmnumber{ #2}: \thmnote{#3}}%
\theoremstyle{break}
\newtheorem{strategy}{Strategy}

\newcommand{\Aff}{\ensuremath{a}\xspace}
\newcommand{\AP}{\ensuremath{\mathrm{AP}}\xspace}
\newcommand{\gap}{\textsc{GAP}\xspace}
\newcommand{\C}{\ensuremath{\mathcal{C}^k_j}\xspace}
\newcommand{\Cs}{\ensuremath{|\C|}\xspace}
\newcommand{\A}{\ensuremath{\mathcal{A}^k}\xspace}
\newcommand{\T}{\ensuremath{\mathcal{T}^k}\xspace}
\newcommand{\OS}{OS/\(\gamma\)\xspace}
\newcommand{\iina}{\ensuremath{i \in \A}}
\newcommand{\jint}{\ensuremath{j \in \T}}
\newcommand{\secref}[1]{Section~\ref{#1}\xspace}
\newcommand{\tabref}[1]{Table~\ref{#1}\xspace}
\newcommand{\figref}[1]{Figure~\ref{#1}\xspace}
\newcommand{\myeqref}[1]{(Equation~\ref{#1})\xspace}
\newcommand{\authorcite}[1]{\citeauthor{#1}~\shortcite{#1}\xspace}

\begin{document}

\title{Multi-Cycle Assignment Problems with Rotational Diversity}

\author{\name Helge Spieker \email helge@simula.no \\
       \name Arnaud Gotlieb \email arnaud@simula.no \\
       \addr Simula Research Laboratory,\\
       Lysaker, Norway
       \AND
       \name Morten Mossige \email morten.mossige@uis.no \\
       \addr University of Stavanger, Stavanger, Norway\\
       ABB Robotics, Bryne, Norway}

\author{Helge Spieker$^1$ \and Arnaud Gotlieb$^1$ \and Morten Mossige$^{2,3}$}
\date{%
	$^1$ Simula Research Laboratory, Lysaker, Norway\\%
    $^2$ University of Stavanger, Stavanger, Norway\\%
    $^3$ ABB Robotics, Bryne, Norway\\%
    \{helge,arnaud\}@simula.no, morten.mossige@uis.no\\[2ex]%
}

\maketitle

\begin{abstract}
    Multi-cycle assignment problems address scenarios where a series of general assignment problems has to be solved sequentially. 
    Subsequent cycles can differ from previous ones due to changing availability or creation of tasks and agents, which makes an upfront static schedule infeasible and introduces uncertainty in the task-agent assignment process.
    We consider the setting where, besides profit maximization, it is also desired to maintain diverse assignments for tasks and agents, such that all tasks have been assigned to all agents over subsequent cycles.
    This problem of multi-cycle assignment with rotational diversity is approached in two sub-problems: The outer problem which augments the original profit maximization objective with additional information about the state of rotational diversity while the inner problem solves the adjusted general assignment problem in a single execution of the model.
    We discuss strategies to augment the profit values and evaluate them experimentally.
    The method's efficacy is shown in three case studies: multi-cycle variants of the multiple knapsack and the multiple subset sum problems, and a real-world case study on the test case selection and assignment problem from the software engineering domain.
\end{abstract}

\section{Introduction}


General assignment problems are well-studied in
artificial intelligence and can be solved efficiently. 
Their goal is to assign a set of weighted tasks to a set of agents,
such that capacity constraints are satisfied and a profit function is maximized.
These problems are relevant in a broad context, of which many consider some form of rotation.
In aircraft rotation~\cite{Clarke1996} or machine scheduling~\cite{Ma2010},
rotation mechanisms allow to keep maintenance schedules or optimize usage
patterns of machinery.
In nurse rostering~\cite{Chiaramonte2008,Azizi2010} and workforce scheduling~\cite{Ernst2004,Musliu2018},
rotation is relevant to avoid boredom, fatigue and prolonged high workloads or to cover a constrained shift system.
It is not always possible to address such scheduling requirements upfront,
albeit for personnel availability due to vacations and
sickness leaves, changing demand patterns, or short-term planning horizons for other reasons.
Conclusively, it can be necessary to include rotation mechanisms in scenarios of iterative and recurring planning due to problem constraints and requirements.

This paper addresses multi-cycle assignment problems, where there is uncertainty regarding the availability of tasks and agents, under the additional goal to rotate assignments from tasks to agents over successive cycles.
Tasks and agents can be unavailable for one or several cycles without previous
notice or information about their next availability.
We refer to the subsequent diverse assignments as \emph{rotational diversity}.
A full example of the problem and our solution is given in Section~\ref{sec:affinity}.

It should be noted, that this work defines \textit{rotational diversity} in a temporal manner.
That is, the solution to multiple subsequent instances of a
problem has to differ in the assignments made.
This is different from the notion of solution diversity,
where it is desirable to find multiple distinctly diverse solutions to one
instance of a problem \cite{Glover2000,Hebrard2005,Trapp2015,Petit2015}.

We develop an method, that combines profits and affinities,
a metric to describe the state of rotation, into a single optimization criterion.
Solving this model incrementally, that is, at each cycle,
allows to control rotational diversity.
A central component for this control is the strategy, that defines how profits
and affinities are combined.
Processwise, the method is split into two sub-problems: 
The outer problem which augments the original profit maximization objective with additional information about the state of rotational diversity, while the inner problem solves the adjusted general assignment problem in a single execution of the solver for the assignment problem.

Part of the technical contribution is the presentation of five strategies for this combination of values.
All of these strategies can be further combined with a \textit{Limited Assignment} extension, that restricts the possible assignments between tasks and agents with the goal to more rigorously enforce the solver to produce diverse solutions. 
Using Limited Assignment increases the ability to maintain rotational diversity, but with a trade-off in profit.

As part of the experimental evaluation, three case studies are considered.
The first case study is a multi-cycle extension of the multiple knapsack problem,
and in the second, the multiple subset sum problem is considered in a multi-cycle environment (MCMSSP).
The third case study is a real-world case study of test case selection and assignment problem (TCSA), originating from the software engineering domain.
Our results show that in all case studies rotational diversity can be effectively maintained by the introduced method, while sacrificing only a small percentage of the original goal of profit maximization, e.g. less than 4\,\% in TCSA.

In previous work \cite{Spieker2019}, we introduced the problem of rotational diversity in multi-cycle assignment problems and presented initial strategies to address the problem for the first time.
This paper builds upon the existing results with additional insights and explanations, as well as an extended experimental evaluation.
We furthermore introduce the Limited Assignment, which enforces diversity through manipulating the compatibility between tasks and agents.
This extended strategy can be combined with any of the previous approaches and shows to be effective to further improve the rotational diversity in our experiments.

The remainder of this paper is structured as follows:
\secref{sec:relwork} gives an overview on the related work in the area of general assignment problems, assignments under uncertainty and alternative approaches to our method, then \secref{sec:problem} introduces and formalizes the problem of multi-cycle assignment with rotational diversity.
Our method to approach rotational diversity is presented in \secref{sec:approach} along with the six evaluated strategies.
In \secref{sec:experiments}, we perform an experimental evaluation on two case studies before concluding the paper with a final discussion in \secref{sec:conclusion}.

\section{Related Work}
\label{sec:relwork}

The general multi-cycle assignment problem is a variant of the
\emph{General Assignment Problem}~(GAP)~\cite{Ross1975,Martello1990,Pentico2007}. 
A set of tasks, each
associated with a profit and a weight, has to be assigned to a set of agents
with limited capacity. The goal is to maximize (or minimize) the summed profits
of the assigned tasks, while the weights do not exceed the agent capacities. Not
all tasks are mandatory to be assigned. Profits and weights can vary between agents.
The classical assignment problem formulates a cost minimization objective,
although maximization, which we use throughout this work, is also commonly found
in problem variants.

In this paper, we formulate rotational diversity in terms of the broad class of general assignment
problems, as our contribution is steered towards the general rotation mechanism.
The closest problem variant is the group of knapsack problems. 
One or multiple agents have to be filled in to maximize the value of the selected tasks~\cite{Martello1987}. 
A multi-cycle knapsack variant is presented in~\cite{Faaland1981}, 
although only the unassigned items from previous cycles are available in subsequent cycles.

Assignment rotation is found in job rotation scheduling \cite{Musliu2018}.
Here, a common goal is to find schedules and work assignments for
humans to avoid fatigue, boredom~\cite{Bhadury2006} or accidents~\cite{Carnahan2000},
or to evenly distribute shifts to personnel~\cite{Bard2005,Ayough2012}.
This is often solved by a fixed schedule,
where the assignment between workers and their tasks frequently changes.
While there is existing work on repairing schedules in case of disruptions or stochastic elements \cite{Bidot2009}, for example in the application of university timetabling \cite{Lindahl2019} or repair scheduling \cite{Bajestani2013}, these approaches require a defined planning horizon with one, possibly large, optimization problem in the beginning and a number of follow-up problems in case of disruptions.
A specific application is further scheduling and assignment under the awareness of uncertainties within the given data \cite{Herroelen2005}, especially for varying task durations or weights, which can be addressed by robust local search \cite{Fu2012} or stochastic optimization models \cite{Song2019}.
In our approach, we solve subsequent assignment problem without a fixed planning horizon.
That is, we do not fix one assignment over multiple cycles, but have to repeatedly create individual assignments at each cycle due to changing availability of agents and tasks.
In relation to the terminology introduced in \cite{Bidot2009}, our presented method is a progressive technique, where each part of the overall schedule is created sequentially at each cycle. However, their framework and terminology addresses the ability to repair schedules under uncertainty and does not include the desire to actively introduce diversity in the assignments between each time-step or cycle.

Opposite to diverse rotations is the concept of \emph{persistence} in robust optimization~\cite{Bertsimas2006,Morrison2010}.
Persistence \cite{Brown1997,Bertsimas2006} considers finding stable assignments
during optimization, such that improvements in the solution objective only cause
small changes in the variable assignment of the solution. 
By maximizing the affinity between tasks and agents, we can adjust the presented
method to support persistent instead of diverse assignments in subsequent
iterations of a problem using similar techniques.
We note that the concept of affinity, which we introduce in Section~\ref{sec:mcgap}, can be
transferred to multi-cycle problems with persistence.

Fair allocations, which maximize a social welfare function, are considered in
game theory research. 
Mechanisms for the resource distribution include \emph{combinatorial auctions} and
\emph{exchanges}~\cite{DeVries2003,Endriss2006}.
Both have shown to result in a balanced and fair distribution of
resources, although it is complex to determine which resources to offer in an
auction or exchange and who is the resulting winner~\cite{Sandholm2003}.
Recent works further discuss aspects of repeated matching between tasks and agents,
under consideration of dynamic preferences and fairness~\cite{Hosseini2015}, or
repeated matching of previously unmatched tasks~\cite{Anshelevich2013}.
Because combinatorial auctions and exchanges can be decentralized,
these techniques are commonly used for resource allocation in multi-agent systems~\cite{Liu2008,Nongaillard2008}.

In this work, we do not directly solve the \gap,
but instrument a general solver to maintain a fair distribution of tasks to agents.
An alternative is a system where agents exchange tasks among them to achieve rotation.
However,
preliminary experiments showed this approach to be inferior to the one presented.
The evaluated exchange model first focused on profits only,
and afterwards aimed for a fair rotation by allowing one-task exchanges between agents.
It showed that a high-quality \gap solution limits the number of choices for one-task
exchanges and only minimal improvements in rotational diversity occur.

\section{Problem Description}
\label{sec:problem}

We first introduce the multi-cycle assignment problem as a combination of two
sub-problems, which we further define afterwards.
Then, we discuss the characteristics of the class of general assignment problems
at the core of our approach.
Finally, we formulate and discuss the requirements for maintaining rotational diversity over multiple cycles.

In a multi-cycle assignment problem, every cycle is a distinct planning unit,
because, due to the availability of tasks and agents, planning ahead is not possible.
Therefore,
rotational diversity has to be considered at every cycle.
This separates the overall problem into two partial sub-problems (as visualized in~\figref{fig:process}):
First,
the \emph{inner problem} is to solve an independent \gap in each cycle \(k\).
The \gap selects a subset of the available tasks while maximizing the sum of
their values.

Second,
the \emph{outer problem} aims to maintain a diverse assignment between tasks to agents,
meaning that the tasks are frequently assigned to all compatible agents over subsequent cycles.
As a mechanism for this balance,
we utilize the \emph{affinity} between a single task and each agent, and the
\emph{affinity pressure} as a metric to evaluate the whole set of tasks
and agents. The balancing mechanism between profit optimization and rotation of
tasks, is called a \emph{strategy}.

The inner problem, as well as both the affinity and the affinity pressure will be further defined and introduced in the following section.
As part of our method, we introduce six strategies for achieving
rotational diversity in Section~\ref{sec:strategies}.

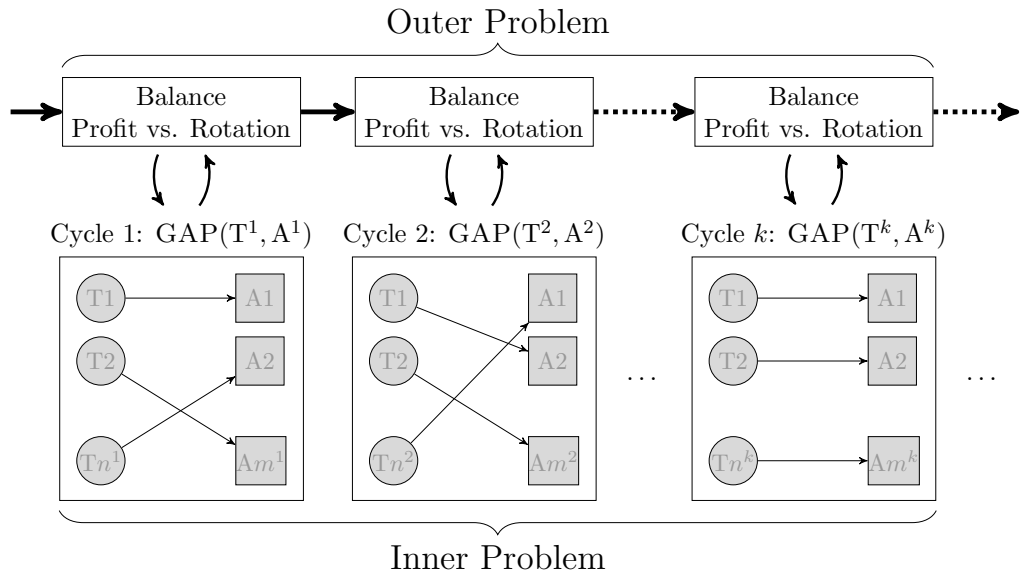
\begin{figure}
\centering
\resizebox{0.95\textwidth}{!}{\tikzstyle{line} = [draw, -latex']
\tikzstyle{cyclebox} = [rectangle, draw, minimum width=3.55cm, minimum height=3.55cm]
\tikzstyle{arrowedge} = [bend right, shorten <=0.75ex, shorten >=0.75ex]


\begin{tikzpicture}[
	mycircle/.style={
         circle,
         draw=black,
         fill=gray,
         fill opacity = 0.3,
         text opacity=1,
         inner sep=0pt,
         minimum size=20pt,
         font=\small},
      gaparrow/.style={-stealth'},
      node distance=0.2cm and 1.6cm
      ]

\node[cyclebox] (cycle1) {};
\node[above, inner sep=3pt] at (cycle1.north) (t1label) {Cycle $1$: \(\gap(\mathrm{T}^1, \mathrm{A}^1)\)};
\node[rectangle, draw, above=1.6cm of cycle1, align=center] (t1) {Balance\\Profit vs. Rotation};

\node[cyclebox, right=0.7cm of cycle1] (cycle2) {};
\node[above, inner sep=3pt] at (cycle2.north) (t2label) {Cycle $2$: \(\gap(\mathrm{T}^2, \mathrm{A}^2)\)};
\node[rectangle, draw, above=1.6cm of cycle2, align=center] (t2) {Balance\\Profit vs. Rotation};

\node[right=0.3cm of cycle2] (dots) {$\dots$};

\node[cyclebox, right=0.3cm of dots] (cyclek) {};
\node[above, inner sep=3pt] at (cyclek.north) (tklabel) {Cycle $k$: \(\gap(\mathrm{T}^k, \mathrm{A}^k)\)};
\node[rectangle, draw, above=1.6cm of cyclek, align=center] (tk) {Balance\\Profit vs. Rotation};

\node[right=0.3cm of cyclek] (dots2) {$\dots$};

\node[left=0.75cm of t1] (tstart) {};
\node[right=1.25cm of tk] (tend) {};

\path[-stealth', line width=2pt] (tstart) edge (t1) {};
\path[-stealth', line width=2pt] (t1) edge (t2) {};
\path[-stealth', line width=2pt, dotted] (t2) edge (tk) {};
\path[-stealth', line width=2pt, dotted] (tk) edge (tend) {};

\path[-stealth', line width=1pt] (t1) edge [arrowedge] (t1label);
\path[-stealth', line width=1pt] (t1label) edge [arrowedge] (t1) {};

\path[-stealth', line width=1pt] (t2) edge [arrowedge] (t2label);
\path[-stealth', line width=1pt] (t2label) edge [arrowedge] (t2) {};

\path[-stealth', line width=1pt] (tk) edge [arrowedge] (tklabel);
\path[-stealth', line width=1pt] (tklabel) edge [arrowedge] (tk) {};

\node[mycircle, below right=0.5cm of cycle1.north west] (t11) {T$1$};
\node[mycircle, below=of t11] (t12) {T$2$};
\node[mycircle, fill opacity=0, draw=none, below=0.01cm of t12] (t1d) {\vdots};
\node[mycircle, below=0.01cm of t1d] (t1n) {T$n^1$};
\node[mycircle, rectangle, right=of t11] (a11) {A$1$};
\node[mycircle, rectangle, below=of a11, right=of t12] (a12) {A$2$};
\node[mycircle, fill opacity=0, draw=none, below=0.01cm of a12, right=of t1d] (a1d) {\vdots};
\node[mycircle, rectangle, below=0.01cm of a1d, right=of t1n] (a1m) {A$m^1$};
\draw [gaparrow] (t11) edge (a11);
\draw [gaparrow] (t12) edge (a1m);
\draw [gaparrow] (t1n) edge (a12);

\node[mycircle, below right=0.5cm of cycle2.north west] (t21) {T$1$};
\node[mycircle, below=of t21] (t22) {T$2$};
\node[mycircle, fill opacity=0, draw=none, below=0.01cm of t22] (t2d) {\vdots};
\node[mycircle, below=0.01cm of t2d] (t2n) {T$n^2$};
\node[mycircle, rectangle, right=of t21] (a21) {A$1$};
\node[mycircle, rectangle, below=of a21, right=of t22] (a22) {A$2$};
\node[mycircle, fill opacity=0, draw=none, below=0.01cm of a22, right=of t2d] (a2d) {\vdots};
\node[mycircle, rectangle, below=0.01cm of a2d, right=of t2n] (a2m) {A$m^2$};
\draw [gaparrow] (t21) edge (a22);
\draw [gaparrow] (t22) edge (a2m);
\draw [gaparrow] (t2n) edge (a21);

\node[mycircle, below right=0.5cm of cyclek.north west] (t31) {T$1$};
\node[mycircle, below=of t31] (t32) {T$2$};
\node[mycircle, fill opacity=0, draw=none, below=0.01cm of t32] (t3d) {\vdots};
\node[mycircle, below=0.01cm of t3d] (t3n) {T$n^k$};
\node[mycircle, rectangle, right=of t31] (a31) {A$1$};
\node[mycircle, rectangle, below=0.01cm of a31, right=of t32] (a32) {A$2$};
\node[mycircle, fill opacity=0, draw=none, below=0.01cm of a32, right=of t3d] (a3d) {\vdots};
\node[mycircle, rectangle, below=0.01cm of a3d, right=of t3n] (a3m) {A$m^k$};
\draw [gaparrow] (t31) edge (a31);
\draw [gaparrow] (t32) edge (a32);
\draw [gaparrow] (t3n) edge (a3m);

\draw [decorate,decoration={brace,amplitude=10pt,mirror,raise=4pt},yshift=0pt]
(cycle1.south west) -- (cyclek.south east) node [black,midway,below,yshift=-0.5cm] {\Large Inner Problem};

\draw [decorate,decoration={brace,amplitude=10pt,raise=4pt},yshift=0pt]
(t1.north west) -- (tk.north east) node [black,midway,above,yshift=0.5cm] {\Large Outer Problem};

\end{tikzpicture}}
\caption{Multi-Cycle Assignment Problem: At each cycle an independent \gap has to be solved, which includes an optimization objective, e.g. maximizing the sum of profits of all assigned tasks. The sets of available tasks and agents can vary between cycle due to different availability.}
\label{fig:process}
\end{figure}

\subsection{Multi-Cycle General Assignment Problem} \label{sec:mcgap}
The general assignment problem, \(\gap(\T, \A)\),
receives as inputs the tasks and agents available at cycle \(k\).
The set of agents, \A, consists of \(m\) integers \(i\), each with a fixed \emph{capacity}, \(b_i\),
and the set of tasks, \T, consists of \(n\) integers \(j\).
Both sets are given at each cycle and can unpredictably change from cycle \(k\)
to \(k+1\).

The relation between a task and an agent has three fixed attributes:
both the \emph{profit} \(p_{ij}\) and the \emph{weight} \(w_{ij}\) are externally fixed and describe the benefit respectively the
resource demand of task \(j\) when assigned to agent \(i\).
Each task further has a set of \emph{compatible agents}, \C, that it can be assigned to.

The \emph{affinity} \(a_{ij}\) is not fixed, but changes between cycles.
The affinity numerically describes the preferred assignments from tasks to agents,
with higher values giving a higher preference for a task to be assigned to that agent.
It is not given as a problem input parameter, unlike the profit, weight and compatibility,
but it is determined as part of method to maintain rotational diversity.

Additionally, we refer to \emph{values} in the context of the optimization objective of the GAP.
Here, the value \(v_{ij}\) is a combination of profits and affinities,
a way to balance profit- and rotation-oriented assignments.
For a standard assignment problem without affinities,
the values equal the profits.

The affinity between a task and an agent, \(\Aff_{ij}\), is the number of cycles
since the last assignment of task \(j\) to agent \(i\).
The affinity quantifies the preference of a task to be assigned to certain agents during the next cycles.
The affinity pressure is the maximum of all affinities in the set of tasks.
Both the affinity and the affinity pressure will be further discussed after a
definition of the inner assignment problem.

\begin{definition}{Multi-Cycle General Assignment Problem}
    \begin{align}
    \text{Maximize } &\sum_{\iina} \sum_{\jint} x_{ij}v_{ij} \label{eq:gap_max}\\
    \text{subject to }
    &\sum_{\jint} x_{ij}w_{ij} \leq b_i, &&\forall\,\iina\label{eq:capacity_limit}\\ %
    &\sum_{\iina} x_{ij} \leq 1, &&\forall\,\jint\label{eq:assign_task_once} %
    \end{align}
    with
    \begin{align}
    &k: \text{Index of the current cycle}\nonumber\\
    &\A: \text{A set of integers $i$ labeling $m$ agents}\nonumber\\
    &\T: \text{A set of integers $j$ labeling $n$ tasks}\nonumber\\
    &b_i: \text{Capacity of agent } i\nonumber\\
    &v_{ij}: \text{Value of task } j \text{ when assigned to agent } i\label{eq:values}\\
    &w_{ij}: \text{Weight of task } j \text{ on agent } i\nonumber\\
    &x_{ij}:
    \begin{cases}
      1 & \text{Task $j$ is assigned to agent $i$} \land i \in \C\\
      0 & \text{otherwise}
    \end{cases}\label{eq:assignment}
    \end{align}
\end{definition}

The problem's objective is to maximize the total sum values of the assigned
tasks~\myeqref{eq:gap_max}. Each agent can hold multiple tasks up to its
resource limit~\myeqref{eq:capacity_limit} and each task is assigned to at most
one agent~\myeqref{eq:assign_task_once}. The assignment of tasks to agents is
constrained by compatibility constraints~\myeqref{eq:assignment}, such that each
task can only be placed on a subset of agents.

We state a very general GAP formulation,
although our proposed approach is able to handle different GAP variants.
The most important and required properties of the formulation are:
a) the possibility to have different values per task and agent~\myeqref{eq:values},
and b) the value maximization objective~\myeqref{eq:gap_max}.

GAP is NP-hard as it reduces to the NP-hard one-dimensional knapsack
optimization problem~\cite{Karp1972}.

\subsection{Rotational Diversity}
\label{sec:rotdiv}

To maintain rotational diversity,
it is necessary to control the affinities between tasks and agents.
As an indicator, the affinity pressure must not grow too high,
which can be avoided by a diverse rotation between tasks and agents.

As part of solving the outer problem,
it is necessary to balance profit maximization and reducing affinities by rotating
the assignments from tasks to agents.
Additional complexity stems from the fact, that at each cycle different sets of
agents and tasks are available and the assignment can only take the current cycle into account.

The optimization in the outer problem could be solved by an exhaustive search of
possible combinations between profits and affinities,
such that an optimal solution can be found.
In practice, this is infeasible,
as it requires to solve the computationally expensive inner \gap problem
multiple times before deciding for the final solution.

\section{Maintaining Rotational Diversity}
\label{sec:approach}

The central idea for maintaining rotational diversity is the manipulation of the values
contributing to the objective of the inner assignment problem (see \figref{fig:mat_combination}).
This adjustment steers the optimization process towards an assignment which is balancing profit maximization
and making diverse assignments.
The adjustment is made according to a strategy and the state of the available resources,
that is tasks and agents available in the current cycle, and their affinities.

\begin{figure}[t]
\begin{equation*}
\begin{array}{>{\centering\arraybackslash$}p{2.2cm}<{$}>{\centering\arraybackslash$}p{0.15cm}<{$}>{\centering\arraybackslash$}p{2.2cm}<{$}>{\centering\arraybackslash}p{0.05cm}>{\centering\arraybackslash$}p{2.2cm}<{$}}
\text{Profits} & & \text{Affinities} & & \text{Values}\\ 
\left[
\begin{smallmatrix}
  p_{1,1} & \cdots & p_{1,n} \\
  \vdots  & \ddots & \vdots  \\
  p_{m,1} & \cdots & p_{m,n} 
\end{smallmatrix}
\right] &
\bigcirc &
\left[
 \begin{smallmatrix}
  a_{1,1} & \cdots & a_{1,n}\\
  \vdots  & \ddots & \vdots\\
  a_{m,1} & \cdots & a_{m,n} 
 \end{smallmatrix}
\right] &
= &
\left[
 \begin{smallmatrix}
  v_{1,1} & \cdots & v_{1,n} \\
  \vdots  & \ddots & \vdots  \\
  v_{m,1} & \cdots & v_{m,n} 
 \end{smallmatrix}
\right]
\end{array}
\end{equation*}
\caption{In the outer problem, profits $p$ and affinities $a$ are combined by a strategy \(\bigcirc\) into single
  values $v$. These values are used to optimize the \gap in the inner problem.}
\label{fig:mat_combination}
\end{figure}

Before introducing different adjustment strategies,
we describe the mechanism to calculate the affinities and the affinity pressure, and the relevance of their values.

\subsection{Assignment Diversity}
\label{sec:affinity}

To achieve rotation of tasks over agents in subsequent cycles,
the cycle-specific assignment problem needs an incentive to assign a task to a
different agent than in previous assignments.

This incentive is described by the notion of affinities between tasks and agents,
describing how important an assignment of a task to an agent is to achieve high rotational diversity.
A low affinity value corresponds to a recent assignment from the task to the agent,
whereas a high affinity indicates the necessity to make this assignment again soon.

The affinities are determined by \emph{Affinity Counting}.

\begin{definition}{Affinity Counting}
\label{def:affinity_counting}

\begin{subnumcases}{\Aff_{ij}^{k} =}
    0                     & if $i \notin \C$\label{eq:ap_incomp} \\
    1                     & if $k = 1 \lor x_{ij}^{k-1} = 1$\label{eq:aff_base} \\
    \Aff_{ij}^{k-1} + 1 & if \(\iina \land\,\jint\) \label{eq:aff_inc}           \\
    \Aff_{ij}^{k-1}     & otherwise\label{eq:aff_unavail}
\end{subnumcases}

\textit{Affinity Counting} counts the number of cycles since the last assignment
from task \(j\) to agent \(i\),
starting from 1 at the first cycle or the last assignment~\eqref{eq:aff_base}.
If a task and agent are incompatible, the affinity is always 0~\eqref{eq:ap_incomp}.
At cycle \(k\),
the affinity increases for non-selected,
but possible assignments in the previous cycle \(k-1\)~\eqref{eq:aff_inc}\eqref{eq:aff_unavail}.
\end{definition}

Naturally, the affinity values increase over time as each task can only be
assigned to one of the compatible agents in each cycle.
This growth is anticipated and acceptable to a certain degree,
while at the same time,
growing affinities show the need to make the corresponding assignment soon.

To monitor the overall state of rotational diversity,
we define the \emph{Affinity Pressure} metric.

\begin{definition}{Affinity Pressure (\AP)} \label{def:ap}

\noindent The Affinity Pressure is defined per cycle \(k\) and task \(j\):
\begin{equation*}
\AP_{j}^{k} = \frac{\sum_{\iina}\Aff_{ij}^{k}}{\Cs} - \frac{\Cs+1}{2}
\end{equation*}

It is the scaled difference between the actual and ideal
affinities, as described below.
For the \AP calculation,
only the task and agents available in that cycle are considered. 
Hence, tasks and agents can be added or completely removed without affecting the \AP values of the remaining tasks.
\end{definition}

In an ideal rotation setting,
the affinities of a task \(j\) form the set \(\{\, i \;|\; 1 \leq i \leq \Cs\,\}\),
with its sum being the triangular number \(\frac{1}{2} \cdot \Cs \cdot (\Cs+1)\).
As the task is (ideally) assigned in every cycle,
the last assignment has affinity 1,
the previous assignment has affinity 2, and so on.
With \Cs compatible agents,
the longest unassigned task then has affinity \Cs.

However, in a practical rotation setting, 
this perfect rotation is hindered by non-availability and limited capacities of the agents.
To evaluate the state of rotational diversity,
it is, therefore, crucial to consider how long a task has not been
assigned to each agent,
but also, from an agent's perspective,
the time it has not executed certain tasks.

The \AP metric is derived from the difference between the sum of current
affinities and the ideal values.
For comparability and normalization,
it is scaled by the number of possible agents: 
\(\frac{1}{\Cs} \cdot \left[\sum_{\iina}\Aff_{ij}^{k} - \frac{1}{2} \cdot \Cs \cdot (\Cs+1)\right]\)

In this formula, the minuend describes the current affinities relative to the
number of possible agents, the subtrahend the ideal case with fully regular rotation.
A positive excess indicates missed assignments to achieve ideal rotation.
Note that the bottom value of 0 is an ideal value,
which in practice is usually not achievable,
due to selection and limited availability of tasks and agents,
and the necessary selection in the \gap assignment problem.
During the first \Cs cycles,
the \AP for a task is negative,
as initially all affinities equal 1.
After \Cs cycles, the \AP is always \(\geq 0\).

\begin{example*}
Figure~\ref{fig:affinity_example} presents an example of affinities and their
development over four cycles.
In the initial cycle \(1\), all affinities equal 1 (or 0 for incompatible assignments) and there is no preferred
assignment among all possible assignments.

Over the next cycles, tasks T1 and T2 rotate over all compatible agents,
resulting in the \AP value 0 for T1 and T3.
Task T3 does not rotate, but is assigned to agent C in two
subsequent cycles, which increases the affinity for the assignment to agent B
and raises the \AP to 0.5, an indicator for the imbalance of T3.
Note that, in cycle \(3\), T3 is unavailable,
but this does not affect its affinities in cycle \(4\).
\end{example*}

\begin{figure}[t]
\begin{subfigure}[t]{0.28\textwidth}
\centering
\begin{tabular}{lrrr|r}
\toprule
                    & A                               & B                               & C                               & \AP                \\
\midrule
T1                  & \cellcolor{lightgray}1          & 1                               & 0                               & -0.5               \\
T2                  & 1                               & \cellcolor{lightgray}1          & 1                               & -1.0               \\
T3                  & 0                               & 1                               & \cellcolor{lightgray}1          & -0.5               \\
\bottomrule
\end{tabular}
\caption{Cycle \(1\)}
\label{tab:affinity_example_a}
\end{subfigure}
\hfill
\begin{subfigure}[t]{0.2\textwidth}
\centering
\begin{tabular}{rrr|r}
\toprule
A                               & B                               & C                               & \AP                \\
\midrule
1                               & \cellcolor{lightgray}\textbf{2} & 0                               & 0                  \\
\cellcolor{lightgray}\textbf{2} & 1                               & \textbf{2}                      & -0.3               \\
0                               & \textbf{2}                      & \cellcolor{lightgray}1          & 0                  \\
\bottomrule
\end{tabular}
\caption{Cycle \(2\)}
\label{tab:affinity_example_b}
\end{subfigure}
\hfill
\begin{subfigure}[t]{0.2\textwidth}
  \centering
\begin{tabular}{rrr|r}
\toprule
A                               & B                               & C                               & \AP                \\
\midrule
\cellcolor{lightgray}\textbf{2} & 1                               & 0                               & 0                  \\
1                               & 2                               & \cellcolor{lightgray}\textbf{3} & 0                  \\
\tikzmark{start3}0                               & \textbf{3}                      & 1                               & 0.5\tikzmark{end3} \\
\bottomrule
\end{tabular}
\caption{Cycle \(3\)}
\label{tab:affinity_example_c}
\end{subfigure}
\tikz[remember picture] \draw[overlay, line width=1pt] ([yshift=.35em]pic cs:start3) -- ([yshift=.35em]pic cs:end3);
\hfill
\begin{subfigure}[t]{0.2\textwidth}
\centering
\begin{tabular}{rrr|r}
\toprule
A                               & B                               & C                               & \AP                \\
\midrule
1                               & \textbf{2}                      & 0                               & 0                  \\
2                               & \textbf{3}                      & 1                               & 0                  \\
0                               & \textbf{3}                      & 1                               & 0.5                  \\
\bottomrule
\end{tabular}
\caption{Cycle \(4\)}
\label{tab:affinity_example_d}
\end{subfigure}
\caption{Affinities and Affinity Pressure of three tasks T1, T2, T3 and agents
  A, B, C over four cycles (Bold: Ideal; Highlighted: Assignment
  in cycle \(k\); Strikethrough: Task unavailable)}

\label{fig:affinity_example}
\end{figure}

\begin{theorem}
For any set of tasks \T and agents \A with constant availability,
if a task \(j\) is always assigned to one of the agents for which it has the highest affinity,
a perfect rotation is achieved and the Affinity Pressure is 0.
\end{theorem}
\begin{proof}
With \(N\) possible agents, it takes \(N\) cycles to assign a task once to every
agent. The affinity is set to \(1\) after the assignment was made and is
increased by \(1\) at every cycle. After each assignment was made once,
the affinity to the first assigned agent is \(N\) again,
the affinity of the second assigned agent is \(N-1\),
and the affinity of the last assigned agent is \(1\).
The sum of affinities is \(\sum_{\iina}\Aff_{ij}^{k}  = \sum_{i=1}^{N} i = \frac{1}{2}N(N+1)\).
Using Definition~\ref{def:ap}, and because the number of available agents is constantly \(\Cs = N\), it follows \(\AP = 0\).
\end{proof}

\subsection{Strategies}
\label{sec:strategies}

A central strategy balances profit maximization and diverse assignments,
by controlling the combination of profits and affinities into values.
This combination then steers the focus of the single-objective GAP solver.

\begin{algorithm}[t]
  \begin{algorithmic}[1]
    \Function{$\mathrm{ExecuteCycle}$}{$\T, \A$}
    \State {$AP^k \gets \text{Calculate Affinity Pressure }\AP(\T,\;\A)$}
    \State {$\T_{Values} \gets \mathrm{Strategy}(\T,\;\A)$}
    \State {$\mathrm{Assignment} \gets \text{Solve }\mathrm{\gap}(\T_{Values},\;\A) $}
    \State {$\T,\;\A \gets \mathrm{UpdateAffinity}(\T,\;\A,\;\mathrm{Assignment})$}
    \State \Return {$\T,\;\A,\;\mathrm{Assignment}$}
    \EndFunction
  \end{algorithmic}
\caption{Solving a single cycle of the multi-cycle assignment problem under consideration of rotational diversity.} 
\label{alg:schema}
\end{algorithm}

The general optimization scheme for a single cycle is shown in Algorithm~\ref{alg:schema}.
First, the state of the system, given by available tasks and agents,
and the affinity pressure are gathered.
Second,
the task values are derived, and the cycle’s \gap is solved.
Finally,
based on the actual assignments,
the affinities of the available tasks are updated.
This procedure adds little overhead to a process where no  
rotation is considered, as the main computational effort remains in the central \gap.

The selected strategy remains constant for the whole process, i.e. every cycle uses the same strategy.
It is nevertheless possible for the strategy to be adaptive and adjust its behaviour according to current state of tasks, agents, and affinities.
At the beginning of every cycle,
the strategy calculates the profit values,
based on profits, affinities,
and (if required) other information about the current state.
These values are then taken as parameters in the
current cycle's GAP instance.

In the following, we present six strategies to control rotational diversity:

\begin{strategy}[Objective Switch (\OS)] \label{sec:objswitch}
The Objective Switch strategy maintains rotational diversity by monitoring the
affinity pressure, and, if it reaches a threshold \(\gamma\), switches from
profit to affinity values:
\begin{equation*}
v_{ij} \triangleq \begin{cases}
p_{ij} & \text{if } \gamma >  \max_{\jint}\AP_{j}^{k}\\
\Aff_{ij} & \text{otherwise}
\end{cases}\,,\quad \forall\,\iina,\; \jint
\end{equation*}

The threshold $\gamma$ is a fixed, user-defined configuration parameter,
and selected according to the desired trade-off between maximized
profits and high rotational diversity.

The objective switch strategy exchanges the focus of the optimization procedure to specifically address a single optimization goal.
It has the intuition, that it is most effective to focus on the rotation goal as soon as the need, quantified by the affinity pressure and $\gamma$, arises.
\end{strategy}

\begin{strategy}[Product Combination (PC)] \label{sec:productcomb}
In the Product Combination strategy,
profit and affinities are multiplied to form the task values:
\begin{equation*}
  v_{ij} \triangleq p_{ij}^\alpha \cdot \Aff_{ij}^\beta\,,\quad \forall\,\iina,\; \jint
\end{equation*}

The exponents $\alpha$ and $\beta$ allow configuration of the strategy to
emphasize one aspect or to account for different scales of profits and
affinities in specific applications.
In our experiments, we found a standard configuration of $\alpha = \beta = 1$ to be
intuitive and well-performing.
Therefore, this strategy does not require additional configuration,
but allows for adjustments if necessary.

In the PC strategy, there is not active reaction on the overall state of
rotational diversity, as in the \OS strategy,
but higher affinities values implicitly influence the profits and put an
emphasis on tasks with missing rotation.
\end{strategy}

\begin{strategy}[Weighted Partial Profits (WPP)] \label{sec:wpp}
The WPP strategy calculates task values with a weighted sum:
\begin{equation*}
  v_{ij} \triangleq \lambda_{j}^k \cdot \frac{p_{ij}}{\displaystyle \max_{\iina} \max_{\jint} p_{ij}}
  + (1-\lambda_{j}^k) \cdot \frac{\Aff_{ij}}{\displaystyle \max_{\iina} \max_{\jint}\Aff_{ij}}\,,\quad \forall\,\iina,\; \jint
\end{equation*}
The task- and cycle-specific weight parameter \(\lambda_{j}^{k}\) balances the influence of each objective on the
final value \(v_{ij}\).
\(\lambda_{j}^{k}\) is self-adaptive and depends on the ratio between ideal and actual affinities,
similar to the affinity pressure.
When the rotational diversity is high,
the influence of the profits is high, too,
otherwise the affinities have higher influence:
\begin{equation*}
  \lambda_{j}^k = \frac{\frac{1}{2} \cdot \Cs \cdot (\Cs + 1)}{\sum_{\iina} \Aff_{ij}}\,,\quad \forall\,\jint
\end{equation*}

To account for different value ranges, both profits and affinities
are scaled to \([0, 1]\) by their respective maxima.
\end{strategy}

\begin{strategy}[Fixed Objective: Profit (FOP)] \label{sec:fop}
Each task value equals the static profit value: 
\begin{equation*}
    v_{ij} \triangleq p_{ij}\,,\quad \forall\,\iina,\; \jint
\end{equation*}
\end{strategy}
\begin{strategy}[Fixed Objective: Affinity (FOA)] \label{sec:foa}
Each task value equals the affinity value:
\begin{equation*}
v_{ij} \triangleq \Aff_{ij}\,,\quad \forall\,\iina,\; \jint
\end{equation*}
\end{strategy}

FOP and FOA represent special cases of the PC strategy, with $\beta = 0$
respectively $\alpha = 0$.
These strategies are the two most extreme approaches,
because each of them ignores the other goal, albeit profits or affinities.
They serve as comparison baselines to evaluate the trade-offs by the other strategies.

In contrast to the discussed strategies, which manipulate the task values, we consider an additional approach to maintain rotational diversity.
This approach does not only manipulate the task values, but also restricts the possible assignments between tasks and agents.

\subsection{Limited Assignment}
\label{sec:limitedassignment}
The \emph{Limited Assignment} approach explicitly constrains a task \(i\) to be assigned to compatible and available agents with a high affinity value.
This is achieved by artificially limiting the possible assignments through temporarily manipulating the compatibility between tasks and agents.
Limited Assignment does not further manipulate the profit values, but only works on the level of possible assignments, i.e. the task value equals the static profit value: $v_{ij} \triangleq p_{ij}$
Therefore, this approach can be combined with any of the strategies as an additional control mechanism.

Specifically, the compatibility between a task and an agent is removed if the affinity $a_{ij}$ is below a threshold value.
The threshold value $th$ can either be a static, user-defined parameter, or dynamically adapted.
As a heuristic for the threshold value, we propose using the mean affinity between a task and all available and compatible agents, rounded to the next smallest integer:
\begin{equation}
    Threshold_j = \floor*{\frac{1}{\Cs} \sum_{i\,\in\,\C} \Aff_{ij}},\quad \forall\,\jint
\end{equation}


As this approach reduces the search space of possible solutions for the instance, it can lead to the removal of optimal solutions from the solution space.
However, in many settings it is neither necessary nor possible to find the optimal solution, and finding a near-optimal solution in a reduced solution space is sufficient.

To further understand the Limited Assignment approach, we revisit the example for affinity calculation (\figref{fig:affinity_example}) from \secref{sec:affinity}.

\begin{example*}
Figure~\ref{fig:affinity_example_la} presents an example of affinities and their
development over four cycles.
In the initial cycle \(1\), all affinities equal 1 (or -- for incompatible assignments) and there is no preferred assignment among all possible assignments.

\begin{figure}[t]
\begin{subfigure}[t]{0.3\textwidth}
\centering
\begin{tabular}{lrrr|r}
\toprule
                    & A                               & B                               & C                               & Th                \\
\midrule
T1                  & \cellcolor{lightgray}1          & 1                               & 0                               & 1               \\
T2                  & 1                               & \cellcolor{lightgray}1          & 1                               & 1               \\
T3                  & 0                               & 1                               & \cellcolor{lightgray}1          & 1               \\
\bottomrule
\end{tabular}
\caption{Cycle \(1\)}
\label{tab:affinity_example_la_a}
\end{subfigure}%
\hfill%
\begin{subfigure}[t]{0.2\textwidth}
\centering
\begin{tabular}{rrr|r}
\toprule
A                               & B                               & C                               & Th               \\
\midrule
1                               & \cellcolor{lightgray}\textbf{2} & 0                               & 1                  \\
\cellcolor{lightgray}\textbf{2} & X                               & \textbf{2}                      & 2               \\
0                               & \textbf{2}                      & \cellcolor{lightgray}1          & 1                  \\
\bottomrule
\end{tabular}
\caption{Cycle \(2\)}
\label{tab:affinity_example_la_b}
\end{subfigure}%
\hfill%
\begin{subfigure}[t]{0.2\textwidth}
  \centering
\begin{tabular}{rrr|r}
\toprule
A                               & B                               & C                               & Th                \\
\midrule
\cellcolor{lightgray}\textbf{2} & 1                               & 0                               & 1                  \\
X                               & 2                               & \cellcolor{lightgray}\textbf{3} & 2                  \\
\tikzmark{start3la}0                               & \textbf{3}                      & 1                               & 2\tikzmark{end3la} \\
\bottomrule
\end{tabular}
\caption{Cycle \(3\)}
\label{tab:affinity_example_la_c}
\end{subfigure}
\tikz[remember picture] \draw[overlay, line width=1pt] ([yshift=.35em]pic cs:start3la) -- ([yshift=.35em]pic cs:end3la);
\hfill
\begin{subfigure}[t]{0.2\textwidth}
\centering
\begin{tabular}{rrr|r}
\toprule
A                               & B                               & C                               & Th               \\
\midrule
1                               & \textbf{2}                      & 0                               & 1                  \\
2                               & \textbf{3}                      & X                               & 2                  \\
0                               & \textbf{3}                      & X                               & 2                  \\
\bottomrule
\end{tabular}
\caption{Cycle \(4\)}
\label{tab:affinity_example_la_d}
\end{subfigure}
\caption{Effect of Limited Assignment on three tasks T1, T2, T3 and agents A, B, C over four cycles (Th: Threshold value; X: Compatibility temporarily removed through Limited Assignment; Bold: Ideal next assignment; Highlighted: Assignment in cycle \(k\); Strikethrough: Task unavailable)}

\label{fig:affinity_example_la}
\end{figure}

Over the next cycles, tasks T1 and T2 rotate over all compatible agents,
resulting in the \AP value 0 for T1 and T3.
Task T3 does not rotate, but is assigned to agent C in two
subsequent cycles, which increases the affinity for the assignment to agent B
and raises the \AP to 0.5, an indicator for the imbalance of T3.
Note that, in cycle \(3\), T3 is unavailable,
but this does not affect its affinities in cycle \(4\).
\end{example*}

\section{Experimental Evaluation}
\label{sec:experiments}

We consider three problems for evaluation: 
a) a multi-cycle variant (MCMKP)
of the known multiple knapsack problem (MKP) to evaluate trade-offs between the
strategies;
b) a multi-cycle variant (MCMSSP) of the multiple subset sum problem (MSSP) to evaluate the behaviour with a smaller number of agents and low task availability;
c) test case selection and assignment (TCSA) as a real-world case study from the software testing domain to evaluate the practical interest of our approach.

\subsection{Implementation and Setup}
Our strategies and the experimental setup are implemented in Python.
The assignment problem is modeled with MiniZinc~2.0~\cite{Nethercote2007},
following the presented \gap formulation, and is solved with IBM~CPLEX~12.8.0.
Our implementation and all test data is available online at \url{https://github.com/HelgeS/mcap_rotational_diversity}.

We note that there are further domain-specific heuristics and exact algorithms to solve knapsack problems, e.g. \cite{Fukunaga2007}, but as the \gap model and its solver are a black-box to our
strategies, their optimization is not in the scope of our work, and we employ a generic model formulation and solver.
To ensure the solution quality with a reasonable time-contract for the solver,
we compared it on a set of sample instances with mulknap\footnote{\url{http://www.diku.dk/~pisinger/codes.html}},
an exact MKP solver~\cite{Pisinger1999}.
With a 60 second timeout,
CPLEX achieves on average 99.5\,\,\% of the optimal solution calculated by mulknap.

All strategies are run on each scenario with a 60 second timeout for the \gap solver.
The thresholds \(\gamma\) for the Objective Switch strategy are 10, 20, 30, and 40, except for the MCMSSP problem, where we use smaller $\gamma$ of 1, 2, 4, and 10.

We evaluate the full rotation of tasks over agents, both looking at all tasks, and at
each individual task.
One full rotation over all tasks is achieved, when each task was assigned once to all
compatible agents.
The rotation over one task describes how often a task is assigned to
its compatible agents on average.
These numbers can be different.
If few tasks are not rotated,
those forestall full rotations,
but allow other tasks to be frequently rotated.

Furthermore, we compare the achieved profit of the assignments with the profit
of the FOP strategy, which does not consider rotation and only maximizes profit.
As the other experimental parameters are the same and also the same assignment model is used,
FOP simulates the baseline setting without rotation-awareness.

We have considered an additional baseline,
where the full multi-cycle assignment problem is optimized as one single optimization model.
This differs from our method, as each task's and agent's availability is known already in the beginning.
However, due to the exceeding model size, solving the extended \gap model is computationally expensive and did not yield
a comparable solution within 24 CPU hours, which is substantially more than the total
computational cost of successively optimizing individual cycles.
Therefore, we do not further consider this baseline.

\subsection{Multi-Cycle Multiple Knapsack Problem} \label{sec:mcmkp}
MKP is a variant of the 0-1 knapsack problem, and thereby of GAP,
with multiple agents, i.e. knapsacks~\cite{Martello1990,Pisinger1995}.

We extend MKP to a multi-cycle variant (MCMKP) with limited availability of tasks and agents.
In every cycle, the same MKP instance has to be solved under
consideration of the assignments made in previous cycles and changing
availability of tasks and agents.

\subsubsection{Setup}

To generate problem instances, we employ the procedure by
\authorcite{Pisinger1999}, as described in \authorcite{Fukunaga2011},
and extend it to the notion of compatibility and availability.
An instance is generated by first creating random tasks with weights from a uniform distribution ($w_j \sim \mathcal{U}[10,
1000]$).
The profits of the tasks are either uncorrelated, i.e. profits are drawn from the same uniform distribution, 
or weakly correlated, i.e. the profits are calculated by $p_{j} = w_j + \mathcal{U}[-99, +99]\,, \forall\,\jint$.

After generating the tasks, the agents $a_1, a_2, a_i, \dots, a_{m-1}$ are generated and set to 40--60\,\% of the tasks' weight.
An exception is the last agent $a_m$, whose capacity is set such that the total capacity of all agents equals half of
the tasks' demand.
The instance sizes are 30/75, 15/45, and 12/48 agents, respectively tasks. 
For this generation scheme, a ratio \(|\T|/|\A|\) slightly larger than 2
leads to hard instances, while instances of higher ratios become easier to
solve~\cite{Fukunaga2011}.
The number of cycles is three times the number of tasks, to allow multiple assignments between tasks and agents, even if an agent
has only capacity for one task.

A notion of compatibility is implicit in the generation procedure.
Tasks that do not fit into an agent's capacity are automatically incompatible.
However, this skews the number of compatible tasks to those agents with high capacities,
and puts more emphasis onto their assignments.

From all combinations of the four parameters, we generate 24 instances with in total 4032 assignment problems for evaluation.
Every instance is run with each strategy alone and in combination with Limited Assignment.

\begin{table}[t]
  \centering
  \begin{tabularx}{0.98\textwidth}{lllRRRRR}
\toprule
\multicolumn{3}{r}{Available Agents} & 75\% & 75\% & 100\% & 100\% & \\
\multicolumn{3}{r}{Available Tasks} &         75\% &        100\% &         75\% & 100\% & Average \\
\midrule
\multirow{16}{*}{\rotatebox[origin=c]{90}{(a) Pure Strategies}} &
\multirow{8}{*}{\rotatebox[origin=c]{90}{Rotational Diversity}} 
                & OS/10 & \textbf{2.0 (4.3)} &  \textbf{2.0 (4.6)} &  \textbf{3.0 (5.3)} &  \textbf{3.3 (5.9)} &  \textbf{2.0 (5.0)} \\
&& OS/20 &  1.6 (3.8) &  2.0 (4.4) &  1.6 (3.9) &  3.0 (5.2) &  2.0 (4.3) \\
&& OS/30 &  1.0 (3.1) &  1.3 (3.8) &  1.3 (3.3) &  2.3 (4.4) &  1.0 (3.7) \\
&& OS/40 &  0.6 (2.7) &  1.3 (3.3) &  0.6 (3.0) &  1.6 (3.8) &  1.0 (3.2) \\
&& PC  &  0.0 (4.2) &  0.0 (4.4)  &  0.0 (5.4) &  0.0 (5.9) &  0.0 (5.0) \\
&& WPP &  1.3 (4.0) &  1.6 (4.2) &  1.6 (4.9) &  1.6 (5.3) &  1.0 (4.6) \\
\cmidrule{3-8}
&& FOA &  2.3 (4.5) &  2.0 (4.8) &  3.0 (5.7) &  3.3 (6.1) &  2.0 (5.3) \\
&& FOP &  0.0 (1.6) &  0.0 (1.5) &  0.0 (1.9) &  0.0 (2.0) &  0.0 (1.7) \\
\cmidrule{2-8}
& \multirow{8}{*}{\rotatebox[origin=c]{90}{Profit (\% of FOP)}}  & OS/10 &       87.8 &       83.7 &       88.5 &       84.9 &       86.2 \\
&& OS/20 &       90.0 &       84.6 &       92.0 &       86.8 &       88.4 \\
&& OS/30 &       92.9 &       87.1 &       94.2 &       89.7 &       91.0 \\
&& OS/40 &       \textbf{94.9} &       89.2 &       \textbf{95.7} &       \textbf{92.1} &       \textbf{93.0} \\
&& PC &          92.3 &       \textbf{90.5} &       90.5 &       90.5 &       91.0 \\
&& WPP &         88.6 &       83.1 &       93.3 &       88.1 &       88.3 \\
                  \cmidrule{3-8}
&& FOA &         87.0 &       82.8 &       86.8 &       84.1 &       85.1 \\
&& FOP &         100.0 &         100.0 &        100.0 &        100.0 &         100.0 \\
\midrule
\midrule
\multirow{16}{*}{\rotatebox[origin=c]{90}{(b) Limited Assignment}} &
\multirow{8}{*}{\rotatebox[origin=c]{90}{Rotational Diversity}} & OS/10 &  2.0 (4.4) &  2.0 (4.7) &  2.6 (5.3) &  3.3 (6.1) &  2.0 (5.1) \\
&& OS/20 &  1.6 (4.1) &  2.0 (4.6) &  1.6 (4.7) &  3.0 (5.7) &  2.0 (4.8) \\
&& OS/30 &  1.3 (3.7) &  1.6 (4.3) &  1.3 (4.5) &  2.3 (5.2) &  1.0 (4.4) \\
&& OS/40 &  1.0 (3.4) &  1.3 (4.0) &  0.6 (4.2) &  1.6 (4.6) &  1.0 (4.1) \\
&& PC &  0.0 (4.2) &  0.0 (4.5) &  0.0 (5.4) &  0.0 (5.9) &  0.0 (5.0) \\
&& WPP &  \textbf{2.3 (4.4)} &  \textbf{2.3 (4.6)} &  \textbf{3.3 (5.8)} &  \textbf{4.0 (6.0)} &  \textbf{3.0 (5.2)} \\
\cmidrule{3-8}
&& FOA &  2.0 (4.5) &  2.0 (4.8) &  3.0 (5.8) &  3.3 (6.2) &  2.0 (5.3) \\
&& FOP &  0.0 (2.9) &  0.0 (3.0) &  0.0 (3.5) &  0.0 (3.6) &  0.0 (3.3) \\
                  \cmidrule{2-8}
& \multirow{8}{*}{\rotatebox[origin=c]{90}{Profit (\% of FOP)}} & OS/10 &       87.6 &       83.5 &       88.2 &       84.8 &         86.0 \\
&& OS/20 &       89.1 &       84.3 &       90.8 &       86.4 &       87.6 \\
&& OS/30 &       91.8 &       86.3 &       91.6 &       88.7 &       89.6 \\
&& OS/40 &       \textbf{93.0} &       88.3 &       \textbf{92.3} &       \textbf{90.4} &         \textbf{91.0} \\
&& PC &         	92.0 &       \textbf{90.2} &       90.4 &       \textbf{90.4} &       90.8 \\
&& WPP &       	85.5 &       81.2 &       86.3 &       83.2 &       84.1 \\
                  \cmidrule{3-8}
&& FOA &       87.0 &       82.8 &       86.7 &       84.1 &       85.1 \\
&& FOP &       95.9 &       95.6 &       91.3 &       94.5 &       94.3 \\
\bottomrule
\end{tabularx}

  \caption{Results for MCMKP. The best results for each (a) pure strategies and (b) Limited Assignment are marked in bold without considering the FOA/FOP baselines.}
  \label{tab:res_mcmkp}
\end{table}

\subsubsection{Pure Strategies}

We first discuss the rotational diversity results grouped by agent and task
availability, that is in four different groups,
as this is the main differentiating attribute of the MCMKP scenarios.
The results for the pure strategies without limited assignment are shown in the upper half of \tabref{tab:res_mcmkp} and the results for the combination with limited assignment in the lower half.
The profit values of all MCMKP results have been scaled in relation to the best achieved profit, which corresponds to the profit-only optimization without limited assignment. That means, also the results for the strategies with limited assignment in \tabref{tab:res_mcmkp} are scaled in relation to the results without limited assignment.
All strategies rank between the extreme baselines, FOA and FOP, that only focus on one aspect of the problem formulation, either rotation or profit optimization.

In the MCMKP, specifically the \OS strategies are effective to achieve either good rotation or good profit optimization while still having better results in the other optimization goal in comparison to the baselines. However, as the $\gamma$ value is a fixed parameter of these strategies, the strategies cannot effectively balance the two goals for a better result.
With a low $\gamma = 10$, OS/10 shows similar performance than FOA, but with a small improvement on profit optimization in cycles where the overall affinity pressure is low.
With the higher $\gamma = 40$, OS/40 has the highest profit among the strategies, except FOP, and still better rotation results than FOP and, in terms of full rotations, also than PC.

The PC strategy multiplies profits and affinities into one value and achieves competitive results for the MCMKP scenario in terms of profit maximization and average rotations, but it does not yield a schedule with full rotations of task.
This is an indication of one or a few tasks that are blocked from achieving full rotations, e.g. because they have low profit values or limited availability either in themselves or the compatible agents, and are therefore not sufficiently scheduled.
As noted before, the MCMKP generation procedure shifts the compatibility of tasks to the last agent, to which most tasks are compatible.
Still, its capacity is limited and only a few tasks, especially those with a height weight, can be assigned per cycle. MCKMP thereby creates a bottleneck for diverse assignments, which PC does not efficiently overcome.
Strategies which focus more directly on achieving rotational diversity have an advantage in that case.
However, we see this result specifically in the MCMKP scenario, but not for the other case study, which we will discuss below.

WPP shows a similar performance as OS/20 while following a more flexible value combination strategy than switching the objective.
This results in very similar average results over all scenarios, but larger differences, for example in the case of 100\,\% availability for both agents and tasks.
Here, OS/20 achieves $3$ full rotations, but WPP only $1.6$, even though the average rotations are similar with $5.2$ respectively $5.3$.
We can attribute the potential miss of full rotations in WPP to the same causes as PC as both strategies follow a similar approach of combining the profit and the affinities into a single value, whereas \OS uses either the one or the other. 

Furthermore, we analyze the influence of varying availability on the achievable rotations. 
As the setting is such that a selection of tasks has to occur (the resource demand is higher than the resource supply), the availability of a large number of agents has a stronger influence than a high task availability.
However, for making diverse assignments, a high task availability is beneficial. 
This can be seen when comparing the results with 75\,\%/100\,\% and 100\,\%/75\,\% agent respectively task availability. 
The more profit-oriented variants of \OS achieve better rotation in the former than in the latter
case, as \OS switches focus, and potentially the optimization objective of a cycle does not match the availability of the tasks. 
Then, one task might only be present at profit maximization, but not for rotation optimization.

\subsubsection{Limited Assignment}
When considering the strategies in combination with Limited Assignment, where the compatibility between tasks and agents is limited to those pairings with high affinity values, we observe two main effects in the MCMKP case study.
First, reducing the search space for assignments is effective for increasing rotational diversity.
For all strategies, an increase in average rotations can be observed and for all strategies, except PC, also an increase in full rotations.
Even FOA, which already in its pure version optimizes rotations can benefit from the limitations on possible assignments, albeit only to a small degree.
This is an indication that the pure version might make assignments with little benefit, either to optimize the cycle's objective by choosing two low-affinity tasks with small weights over a high-affinity task with higher weight or due to limited time to solve the inner problem.
The biggest performance difference in rotational diversity is observed for WPP, which maintains the highest level of rotational diversity, and actually surpasses even the FOA baseline in terms of full rotations, but not average rotations.

Second, while increasing the maintainability of rotational diversity through limited the search space, solutions with a high profit value are potentially removed and the total profit is expected to be reduced when using limited assignment.
This effect is reflected in the results, but only to a small degree for most strategies.
The largest difference occurs in the scenario with 100\,\% available agents and 75\,\% available tasks.
Here, the profit of WPP is reduced by 7 and FOP by 8.7 percentage points.

\subsubsection{Distribution of Results}
\begin{figure}[t]
    \centering
    \includegraphics{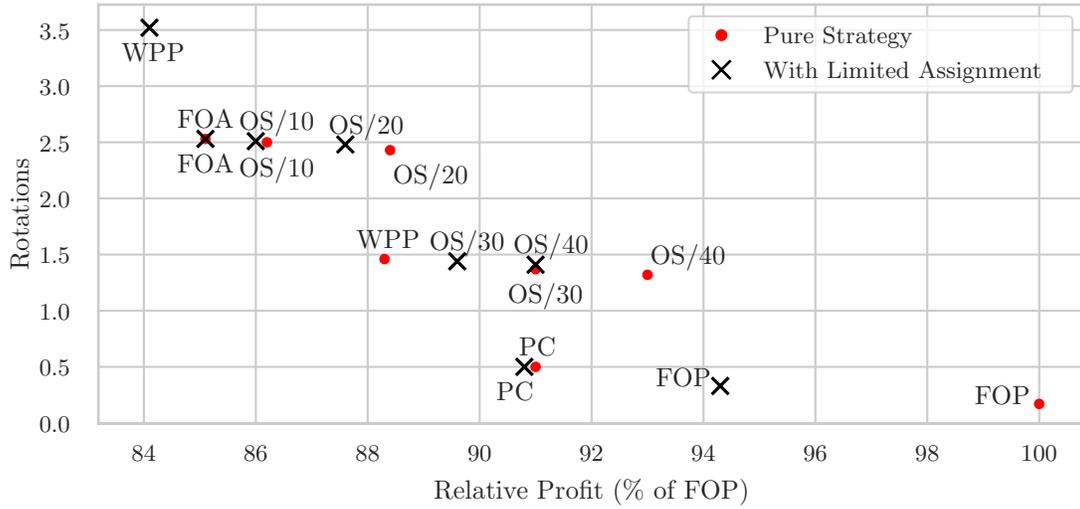}
    \caption{MCMKP: Distribution of the average results for each strategy without (marked with red circle) and with (marked with black X) limited assignment. Rotations are calculated as $\text{Full Rotations} + \text{Average Rotations}/10$.}
    \label{fig:mcmkp_pareto}
\end{figure}

In \figref{fig:mcmkp_pareto}, the distribution of the average profit and rotational diversity over all problem instances is visualized.
Each point corresponds to the average result of the strategy, including those with the additional limited assignment strategy.
The full and average rotations shown in Table~\ref{tab:res_mcmkp} are aggregated into a single value as $\text{Full Rotations} + \text{Avg. Rotations}/10$.
It shows the trade-off between profit maximization and achieving high rotational diversity, as seen by the extreme points FOP and WPP with limited assignment.

The figure visually confirms the previous discussion on the spread of the strategies' performance with WPP on the one end, achieving high rotational diversity, but comparatively low profit-optimization, and FOP on the other end.
In between these extreme points, the other strategies are spread. 
In terms of solution dominance, OS/20 and OS/40 without Limited Assignment show a balanced performance that surpasses the other strategies in either one of the objectives, while not being worse in the other.

However, in general, the role of the trade-off between rotations and profit has to evaluated in the context of the specific use case.
While it can be acceptable to observe a moderate to high decrease in profit in some applications and therefore maintain rotational diversity in a balanced manner, in other use cases profit optimization is the main driver and rotations only of secondary importance.
Depending on how these two goals are valued, or if both are equally important, a more informed decision on the applicable strategy can be made.

\subsection{Multi-Cycle Multiple Subset Sum}

The second experiment evaluates the proposed method on an extension of another assignment problem variant, the multiple subset sum problem (MSSP) \cite{Martello1990,Caprara2000}.
In MSSP tasks are assigned to a number of identical bins with limited capacity, such that the total sum of weights of packed items is maximized.
For our \gap formulation MSSP is implemented as an assignment problem with a number of agents with equal capacity and tasks with same weight and profit.
The multi-cycle variant of MSSP (MCMSSP) requires, as in the previous experiments, to repeatedly solve variations of a MSSP instance under the additional goal to produce diverse assignments over subsequent cycles.

One practical example of this problem is described in \cite{Caprara2000}.
In a production planning setting where a given number of fixed size raw material is available at each day, e.g. steel slabs, it is necessary to distribute the products, which have to be produced, onto the slabs such that only little material is wasted.
The aspect of rotational diversity can be introduced if the raw material is of fixed size, but with different attributes, e.g. color, and all products should be more or less equally often produced in each material type.
If the product demand is further irregular, it is difficult to plan ahead and the product respectively task availability can each between production cycles.

\subsubsection{Setup}

For problem instances, we base our generation procedure again on \cite{Fukunaga2011}. 
In this scenario task weights are uniformly drawn from the low precision distribution $[10, 100]$ and a task's profit equals its weight, $p_j = w_j$ with the same weight for all agents.
All agents have the same capacity, $b_i = c, \forall \iina$, where $c$ is chosen to be 50\,\% of the weight of all tasks, $c = 0.5 \sum_{\jint} w_j$.
Our focus in this experiment is on the influence of task availability on the rotation mechanism.
Therefore, we do not limit the agents' availability or the compatibility between tasks and agents.
All agents are available in every cycle and every task can be assigned to every agent.
However, the total capacity of the agents is limited, such that a selection of tasks has to be made, and the availability of tasks is limited.
We generate instances of 100 cycles with 20 tasks and either 1 or 5 agents and a task availability of either 50 or 75\,\% per cycle.
Due to the smaller number of agents in this experiment, we reduce the values for $\gamma$ in the \OS strategy to 1, 2, 4 and 10.

\subsubsection{Single Agent Scenarios}

\begin{table}[t]
    \centering
    \begin{tabular}{lllrrrr}
\toprule
\multicolumn{3}{l}{Number of Agents}  & 1 & 1 & 5 & 5 \\
\multicolumn{3}{l}{Num. of Tasks/Available}   & 20 / 50\,\% & 20 / 75\,\% & 20 / 50\,\% & 20 / 75\,\% \\
\midrule
\multirow{16}{*}{\rotatebox[origin=c]{90}{(a) Pure Strategies}} &
\multirow{8}{*}{\rotatebox[origin=c]{90}{Rotational Diversity}} 
 & OS/1 &  \textbf{32 (44.9)} &  29 (55.5) &  5 (7.6) &  6 (9.5) \\
&& OS/2 &  28 (44.0) &  36 (51.9) &  4 (7.0) &  5 (9.4) \\
&& OS/4 &  27 (43.7) &  31 (49.0) &  3 (5.7) &  5 (9.1) \\
&& OS/10 &  27 (43.7) &  28 (47.8) &  1 (4.6) &  2 (7.6) \\
&& PC    &  30 (44.1) &  \textbf{41 (49.8)} &  \textbf{5 (7.6)} &  5 (9.0) \\
&& WPP   &  30 (43.7) &  38 (48.4) &  5 (7.4) &  \textbf{6 (9.0)} \\
\cmidrule{3-7}
&& FOA   &  29 (45.4) &  28 (55.5) &  4 (7.8) &  5 (9.5) \\
&& FOP   &  27 (43.7) &  27 (47.8) &  1 (4.4) &  0 (6.6) \\
\cmidrule{2-7}
& \multirow{8}{*}{\rotatebox[origin=c]{90}{Profit (\% of FOP)}} 
 & OS/1 &         97.8 &         95.2 &  96.0 & 93.0 \\
&& OS/2 &         99.5 &         97.1 &  96.5 & 93.2 \\
&& OS/4 &         \textbf{100.0} &         99.1 &       98.2 &       93.9 \\
&& OS/10 &        \textbf{100.0} &          \textbf{100.0} &       \textbf{99.8} &       \textbf{97.7} \\
&& PC    &        99.9 &           99.9 &         97.6 &       95.0 \\
&& WPP   &        99.9 &           99.9 &         99.3 &       94.8 \\
\cmidrule{3-7}
&& FOA &          96.4 &           95.0 &         95.5 &       92.9 \\
&& FOP &         100.0 &          100.0 &        100.0 &      100.0 \\
\midrule
\midrule
\multirow{16}{*}{\rotatebox[origin=c]{90}{(b) Limited Assignment}} &
\multirow{8}{*}{\rotatebox[origin=c]{90}{Rotational Diversity}} 
 & OS/1 &  \textbf{32 (44.9)} &  29 (55.5) &  4 (7.5) &  5 (9.4) \\
&& OS/2 &  28 (44.0) &  36 (51.9) &  4 (7.2) &  5 (9.5) \\
&& OS/4 &  27 (43.7) &  31 (49.0) &  4 (7.0) &  5 (9.3) \\
&& OS/10 &  27 (43.7) &  28 (47.8) &  2 (6.7) &  3 (8.8) \\
&& PC &  30 (44.1) &  \textbf{41 (49.8)} &  \textbf{6 (7.7)} &  5 (9.1) \\
&& WPP &  30 (43.7) &  38 (48.4) &  5 (7.6) &  \textbf{6 (9.1)} \\
\cmidrule{3-7}
&& FOA &  29 (45.4) &  28 (55.5) &  5 (7.7) &  5 (9.5) \\
&& FOP &  27 (43.7) &  27 (47.8) &  2 (6.7) &  0 (8.6) \\
\cmidrule{2-7}
& \multirow{8}{*}{\rotatebox[origin=c]{90}{Profit (\% of FOP)}} 
 & OS/1 &         97.8 &         95.2 &       95.3 &       92.7 \\
&& OS/2 &         99.5 &         97.1 &       96.7 &       92.9 \\
&& OS/4 &          \textbf{100.0} &         99.1 &       98.8 &       93.8 \\
&& OS/10 &          \textbf{100.0} &          \textbf{100.0} &       \textbf{99.5} &       \textbf{97.9} \\
&& PC &         99.9 &         99.9 &       97.3 &       94.5 \\
&& WPP &         99.9 &         99.9 &      96.0 &       92.9 \\
\cmidrule{3-7}
&& FOA &         96.4 &           95.0 &         95.0 &       92.6 \\
&& FOP &          100.0 &          100.0 &       99.5 &       99.8 \\
\bottomrule
\end{tabular}

    \caption{Results for multi-cycle MSSP. The best results for each (a) pure strategies and (b) Limited Assignment are marked in bold without considering the FOA/FOP baselines.}
    \label{tab:res_mcmssp}
\end{table}

The results for the MCMSSP experiments are shown in Table~\ref{tab:res_mcmssp}.
In the scenario with only a single agent, the Limited Assignment approach has no effect, because no compatibility between tasks and the single agent can be removed and it is not part of this approach to completely remove tasks from the optimization problem. 
Therefore, the results for the scenarios with a single agent are identical to the previously discussed Pure Strategies.
Due to the single agent scenario, the number of full rotations is high and expresses how often each tasks has been assigned, as there is no actual rotation over different but just being assigned or not being assigned.
Still, there is a difference visible between the strategies in the diversity of the assignments and task selection.
For example, in the scenario with one agent and a task availability of 75\,\%, the PC strategy with 41 full rotations achieves 40\,\% more rotations compared to OS/10 with 28 full rotations, while maintaining a similar level of profit (99.9\,\% vs. 100\,\%).

In general, in the single agent scenarios, rotational diversity can be maintained with only 0.1\,\% reduction in profits by using the PC or WPP strategy.
Selecting the tasks to be assigned carefully is still important, as it is not possible to assign all available tasks per cycle. 
The utilization of the agent and the percentage of assigned tasks shows no major differences between the strategies.
In the scenario with 50\,\% available tasks, on average 90\,\% of these tasks are assigned, with an agent utilization of 90\,\%, including for the profit strategy. 
With 75\,\% available tasks, the utilization increases to 99\,\% and 65--67\,\% of the available tasks are assigned on average.

\subsubsection{Multiple Agent Scenarios}

When comparing the results for five agents and 20 tasks, the difference is in the task availability within each cycle, which is either 50 or 75\,\%.
Between the results of the pure strategies and Limited Assignment, we see a similar effect as before, such that Limited Assignment is effective to increase full rotations for most strategies with a profit trade-off.

The difference between scenarios with five agents but different task availability, i.e. 50\,\% respectively 75\,\%, shows as a higher task availability increases the possible assignments and thereby the profit gap between rotation-oriented strategies and the profit-only baseline FOP.
A low task availability restricts the possible assignments of all strategies and thereby leads to a smaller gap in profits. Still, employing a rotational diversity strategy allows to control this trade-off.

\subsection{Test Case Selection and Assignment} \label{sec:tcsa}

As a third and final case study, we employ the real-world application of Test Case Selection and Assignment (TCSA) for cyber-physical systems~\cite{Yoo2012},
such as industrial robots.
TCSA usually occurs in Continuous Integration (CI) processes,
where new releases of the robot control software are regularly integrated and released \cite{Mossige2017}.
Typically, CI involves assigning test cases to test agents several times a day.
Comprehensive test suites exist,
but available time and hardware for their execution are limited.
Then it is necessary to distribute a selection of the most relevant test cases over the available agents.
The test case relevance is given by an upstream test case prioritization
process \cite{Rothermel2001,Spieker2017b}.
This priority can be different at each cycle,
due to discovered failures or changes in the system-under-test.
The assignment of tests to agents is constrained by the available time and 
compatibility between test and agent.
In the \gap terminology, the test case priority resembles the profit,
the test's duration the task weight,
and an agent's available time its capacity.

Additionally,
the availability of agents is influenced by maintenance, technical faults,
or short-term usage in other projects,
and the set of test cases changes due to the ongoing development.
If TCSA were to be solved by a static assignment, this changing availability would create a need for frequent updates and schedule repairs.
Therefore, a static schedule is not practically applicable without additional effort.
Instead, to capture the dynamic setting,
an individual selection and assignment has to be made at each cycle.
Enforcing diverse assignments increases the coverage of tasks and agents, and
thereby the confidence into the system-under-test.

\subsubsection{Setup}
We evaluate the strategies on TCSA,
based on actual test data from our industrial partner.
A set of test cases is to be divided among a number of test agents.
The test case selection has to select those tests with the highest
priorities, which is assigned externally,
and to ensure a rotation of tests between agents,
such that a test is frequently executed on all compatible agents.
All test agents have the same capacity,
that is the time available for a test cycle, which is 10 hours.
Due to unique hardware specifications and different functionality,
a test case is compatible with approximately 60\,\% of the test agents.
The runtime of a test case varies from 1 to 21 minutes,
but is identical for each test agent.
In practice, test agents are not exclusively available for testing or might be defect and test cases are temporarily removed from the test suite.
Therefore,
an average of 40\,\% of the agents and 10\,\% of the test cases are unavailable for 3--7 cycles.
In total, we consider four scenarios,
20 agents with 750, 1500, and 3000 test cases,
and 30 agents with 3000 test cases.

\subsubsection{Pure Strategies}

\begin{table}[t]
    \centering
    \begin{tabularx}{0.98\textwidth}{lllRRRRR}
\toprule
\multicolumn{3}{r}{Number of Agents} & 20 & 20 & 20 & 30 & \\
\multicolumn{3}{r}{Number of Tasks} &  750 & 1500 & 3000 & 3000 & Average \\
\midrule
\multirow{16}{*}{\rotatebox[origin=c]{90}{(a) Pure Strategies}}
& \multirow{8}{*}{\rotatebox[origin=c]{90}{Rotational Diversity}} 
                  & OS/10 &  13 (22.0) &  6 (15.5) &  \textbf{3 (9.3)} &  \textbf{3 (8.4)} &  6 (13.8) \\
&& OS/20 &   8 (18.5) &  6 (15.2) &  3 (9.2) &  3 (8.3) &  5 (12.8) \\
&& OS/30 &   6 (17.0) &  5 (14.2) &  3 (9.0) &  3 (8.1) &  4 (12.1) \\
&& OS/40 &   6 (16.4) &  4 (13.1) &  3 (8.9) &  3 (7.9) &  4 (11.6) \\
&& PC      &  \textbf{15 (24.0)} &  \textbf{7 (14.3)} &  3 (8.2) &  3 (7.5) &  \textbf{7 (13.5)} \\
&& WPP   &  14 (24.0) &  7 (14.1) &  3 (7.3) &  3 (7.0) &  6 (13.1) \\
                  \cmidrule{3-8}
&& FOA &  15 (24.3) &  6 (15.6) &  3 (9.5) &  3 (8.5) &  6 (14.5) \\
&& FOP &   5 (15.9) &  0 (10.8) &  0 (7.0) &  0 (4.6) &   1 (9.6) \\
                  \cmidrule{2-8}
& \multirow{8}{*}{\rotatebox[origin=c]{90}{\shortstack[l]{Profit (\% of FOP)}}} 
    & OS/10 &         96.4 &        79.5 &       67.8 &       74.9 &        79.6 \\
&& OS/20 &         98.1 &        80.0 &       68.3 &       75.3 &        80.5 \\
&& OS/30 &         99.0 &        84.7 &       69.0 &       75.9 &        82.2 \\
&& OS/40 &         99.4 &        90.8 &       69.6 &       76.5 &        84.1 \\
&& PC       &         \textbf{99.7} &        \textbf{97.7} &       \textbf{91.1} &       \textbf{96.6} &        \textbf{96.3} \\
&& WPP    &         98.1 &        77.3 &       54.7 &       66.6 &        74.2 \\
                  \cmidrule{3-8}
&& FOA &         96.1 &          79.0 &       67.4 &       74.4 &        79.2 \\
&& FOP &         100.0 &         100.0 &        100.0 &        100.0 &         100.0 \\
\midrule
\midrule
\multirow{16}{*}{\rotatebox[origin=c]{90}{(b) Limited Assignment}}
    & \multirow{8}{*}{\rotatebox[origin=c]{90}{Rotational Diversity}} 
    & OS/10 &  14 (22.5) &  6 (15.5) &  3 (9.4) &  3 (8.4) &  6 (13.9) \\
&& OS/20 &  10 (20.4) &  6 (15.4) &  3 (9.3) &  3 (8.3) &  5 (13.4) \\
&& OS/30 &  11 (20.2) &  5 (14.8) &  3 (9.2) &  3 (8.2) &  5 (13.1) \\
&& OS/40 &  11 (20.3) &  4 (14.2) &  3 (9.1) &  3 (8.1) &  5 (12.9) \\
&& PC       &  14 (23.9) &  7 (14.4) &  3 (8.2) &  3 (7.5) &  6 (13.5) \\
&& WPP    &  \textbf{16 (24.1)} &  \textbf{8 (14.1)} &  \textbf{4 (7.5)} &  \textbf{4 (7.2)} &  \textbf{8 (13.2)} \\
\cmidrule{3-8}
&& FOA &  15 (24.3) &  6 (15.6) &  3 (9.5) &  3 (8.5) &  6 (14.5) \\                
&& FOP &   9 (20.2) &  2 (14.0) &  0 (9.3) &  0 (7.0) &  2 (12.6) \\                  
\cmidrule{2-8}
& \multirow{8}{*}{\rotatebox[origin=c]{90}{Profit (\% of FOP)}}
    & OS/10 &         96.7 &        79.6 &       67.8 &       74.9 &        79.8 \\
&& OS/20 &         99.8 &        80.1 &       68.4 &       75.4 &        80.9 \\
&& OS/30 &          \textbf{100.0} &        87.2 &       68.9 &         76.0 &          83.0 \\
&& OS/40 &          \textbf{100.0} &        94.9 &       69.6 &       76.5 &        85.3 \\
&& PC &         99.6 &        \textbf{97.7} &       \textbf{91.1} &       \textbf{96.6} &        \textbf{96.3} \\
&& WPP &         95.7 &        72.7 &       54.8 &       65.4 &        72.2 \\
\cmidrule{3-8}
&& FOA &         96.1 &        79.1 &       67.3 &       74.4 &        79.2 \\
&& FOP &         100.0 &         100.0 &        100.0 &        100.0 &         100.0 \\
\bottomrule
\end{tabularx}

    \caption{Results for TCSA. The best results for each (a) pure strategies and (b) Limited Assignment are marked in bold without considering the FOA/FOP baselines.}
    \label{tab:res_tcsa}
\end{table}

For the pure strategies, without additional Limited Assignment, the results are shown in the upper half of Table~\ref{tab:res_tcsa}.
In the smallest scenario, a full rotation of all tests over all possible agents is achieved 14--15 times over 365 cycles, i.e. every 24--26 days. 
Here, each task is compatible to circa 12 agents (60\,\%), and 60\,\% of the agents are unavailable for multiple cycles. 
For the larger scenarios with the same number of agents, the number of full rotations reduces approximately linear, but not the number of average rotations per task. 
This shows, that some tests are not evenly rotated and hinder the completion of full rotations. 
With a larger number of agents, the average number of rotations per tasks drops, as there are more compatible agents and more cycles are necessary for a full rotation.

The profits earned from the assignments (see lower half of Table~\ref{tab:res_tcsa})
are close to the FOP baseline for all strategies in the smallest scenario, but
decrease with a higher number of tests, except for PC, which
is able to balance profit maximization and rotation better than
the other strategies and even outperforms FOA for complete
rotations. For PC, the profit trade-off is always less than $10\,\%$, and
on average less than $4\,\%$ in comparison to the profit-oriented FOP.
WPP, who showed compelling results in the previous experiment, achieves similar rotational diversity to PC, but is less effective in profit optimization.

\subsubsection{Limited Assignment}

The results for the strategies in combination with Limited Assignment are shown in the lower half of Table~\ref{tab:res_tcsa}.
As we observed in the experiments on MCMKP, adding Limited Assignment and thereby reducing the space of possible assignments between tasks and agents, increases the rotational diversity. This is also true for TCSA.
For most strategies, the number of full and average rotations is increased, with the exception of PC in the smallest scenario, where the number of rotations is decreased by a small amount.
Interestingly, we do not observe a substantial reduction in assigned profits, but all strategies can maintain similar levels of profits with and without Limited Assignment.

\subsubsection{Distribution of Results}

As for the MCMKP case study, we visualize the average results of each strategy in terms of rotations and relative profit in \figref{fig:tcsa_pareto}. 
For TCSA, the rotations are calculated as $\text{Full Rotations} + \text{Average Rotations}/100$ to account for the high number of average rotations in these scenarios.

\begin{figure}[t]
    \centering
    \includegraphics{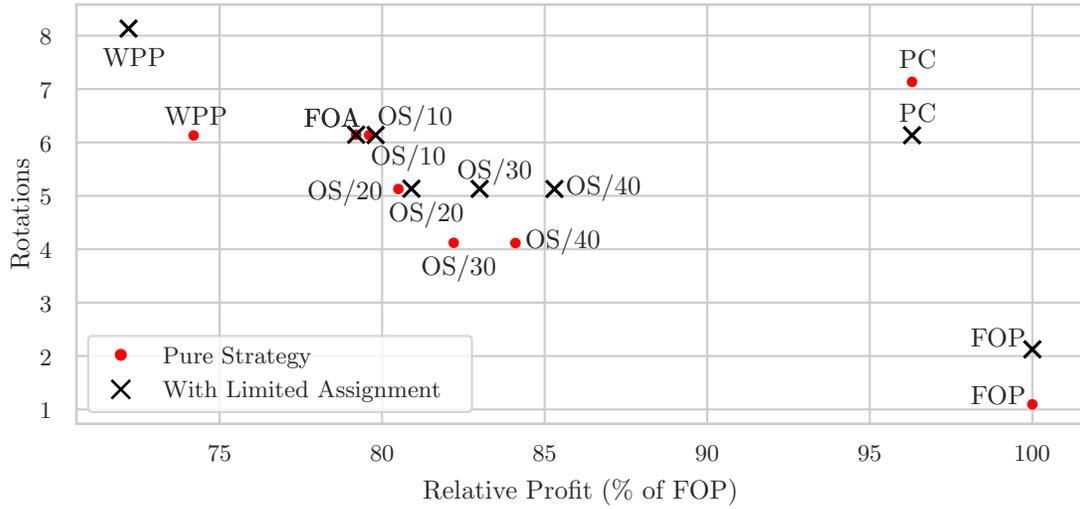}
    \caption{TCSA: Distribution of the average results for each strategy without (marked with red $\circ$) and with (marked with black $x$) limited assignment. Rotations are calculated as $\text{Full Rotations} + \text{Average Rotations}/100$.}
    \label{fig:tcsa_pareto}
\end{figure}

The extreme points being WPP and FOP, as in the MCMKP case study, the different \OS are located closely to each other.
The strong performance of product combination (PC) is also visible in the figure, with pure PC even outperforming the combination of PC and Limited Assignment due to the higher number of full rotations.

\section{Conclusion}
\label{sec:conclusion}
Rotational diversity is the frequent assignment of a task to all its compatible agents over subsequent cycles.
We present a two-part model for its optimization in multi-cycle assignment problems with 
variable availability of tasks and agents:
1) an inner assignment problem,
to optimize the assignment from tasks to agents,
and 2) an outer problem,
to adjust the task values for the maximization objective of the inner problem.

Five strategies,
each having a different approach and trade-offs,
and approaches, one using only the strategies to introduce rotational diversity, the other also controlling the compatibility between tasks agents, are evaluated on three case studies.
Achieving rotational diversity is possible with a profit trade-off of only 4\,\% in the test case selection and assignment case study. 
Both the product combination of profits and affinities,
and the objective switch strategy, that focuses on either profit maximization or diverse assignments, efficiently achieve rotational diversity. 

For applications of this method, we encourage the reader to start from the
product combination strategy. 
It is straightforward to implement and does not require initial configuration, but it can still be adjusted if necessary.

The combination of profits and affinities into a single task value is efficient
for balancing profits and rotation.
This is especially the case in settings where an extended multi-objective
optimization model is not an alternative.
Splitting the problem and its responsibilities allows to use problem-specific,
single-objective solvers for the inner problem, or to use problems with
additional requirements, e.g. precedence constraints or task-dependencies.

In future work, we aim to apply our approach to other settings,
which are derived from the general assignment problem.
These settings can include scheduling problems with precedence constraints and task-dependencies.
Again, rotation of tasks should be achieved without adding substantial computational costs
for rotational diversity.

The affinity metric is also not restricted to
assignment problems with rotational diversity,
but can be transferred to other domains.
One related concept is the idea of persistence, although in
an opposite sense to rotational diversity.
There, a solution should remain stable, even under uncertainties.

\section*{Acknowledgements}
This work is supported by the Research Council of Norway (RCN) through the research-based innovation center Certus, under the SFI program.
The experiments were performed on the Abel Cluster, owned by the University of Oslo and Uninett/Sigma2, and operated by the Department for Research Computing at USIT, the University of Oslo IT-department.

\bibliographystyle{theapa}
\bibliography{library}

\end{document}